%% file: main.tex
\documentclass[letterpaper]{article}
\usepackage[preprint]{aaai2027}
\usepackage[hyphens]{url}
\usepackage{graphicx}
\usepackage{natbib}
\usepackage{caption}
\usepackage{booktabs}
\usepackage{amsmath,amssymb}
\usepackage{amsthm}

\newtheorem{proposition}{Proposition}

\title{SpecAHD: Localize to Specialize for Automated Heuristic Design in Large-Scale Routing Problems}
\author{
Kezhao Lai\equalcontrib,\quad
Yutao Lai\equalcontrib,\quad
Hai-Lin Liu\corresponding
}
\affiliations{
School of Mathematics and Statistics, Guangdong University of Technology, China\\
\texttt{hlliu@gdut.edu.cn}
}

\begin{document}

\maketitle

\begin{abstract}
LLM-based automated heuristic design (AHD) typically scores executable programs on complete instances or within fixed solver components. In large-scale routing problems, localized reconstruction reduces the size of each optimization task, but repair regions within the same incumbent can exhibit substantially different structures. One construction rule must therefore compromise across them. In this paper, we propose \mbox{\textbf{SpecAHD}}, a coupled bilevel framework for within-instance specialization. An upper-level search learns where to expose bounded repair regions, while a lower-level search evolves a complementary repertoire of executable heuristics for the induced repair tasks. The upper-level program determines the repair tasks seen by the lower level, while checked repair outcomes determine how upper-level programs are evaluated. The lower-level objective favors heuristics that perform well on average or solve tasks that the current repertoire handles poorly. For the repair tasks induced by a fixed upper-level program and a fixed lower-level candidate pool, this objective is monotone submodular, allowing greedy repertoire selection with a \((1-1/e)\) approximation guarantee. Across four routing problems and multiple LLM backbones, SpecAHD reduces held-out objective cost by up to \(57.7\%\) against the strongest competing AHD baseline and outperforms the per-instance baseline envelope on most public instances.
\end{abstract}

\section{Introduction}
Large-scale routing problems are often improved through localized reconstruction rather than full re-solving. Large neighborhood search couples two decisions: which part of a feasible incumbent to expose and how to reconstruct it~\cite{shaw1998using,schrimpf2000record,ropke2006adaptive}. An executable constructive heuristic addresses only the second after a partial routing state is specified. Yet the decisions are interdependent: a strong repair rule may offer little on a low-potential region, while a mismatched rule may waste a promising one. The observed cost change therefore confounds region potential with repair suitability, making it unclear which decision an independent search should retain.

LLM-based AHD makes executable heuristics searchable. Representative methods such as FunSearch and Evolution of Heuristics (EoH) generate and refine executable programs through LLM-guided search~\cite{romera2024mathematical,liu2024evolution}. Most existing AHD systems select programs by aggregate performance on complete training instances. EoH-S extends this paradigm by retaining a small set of heuristics with complementary performance across instances~\cite{liu2026eohs}. That set addresses variation across complete instances. Our selection problem recurs within one solution, where each exposed region may call for a different repertoire member.

Recent systems automate neighborhood-search components beyond repair-heuristic generation. LLM-LNS evolves MILP neighborhood-exposure policies, G-LNS co-evolves destroy and repair populations using pairwise rewards, and full-component ALNS searches seven solver modules with a multidimensional elite archive~\cite{llmlns2025,zhao2026glns,yu2026fullcomponent}. These methods capture neighborhood quality, operator compatibility, or component behavior, but do not jointly index local-region structure and executable repair identity for inference. What remains missing is a mechanism that jointly learns which regions to expose and which executable heuristic should repair each exposed structure.

SpecAHD couples the two searches: the upper program induces bounded repair tasks, and the lower repertoire determines their value. Checked region--heuristic outcomes select complementary specialists and populate a Repair Response Archive for inference-time routing. Alternating finite-budget searches update the two sides, and every reconstructed solution is checked before replacing the incumbent.

We make the following contributions:
\begin{enumerate}
\item We introduce SpecAHD, a coupled bilevel AHD framework that moves specialization inside an instance: an upper program exposes bounded regions, and a lower repertoire supplies different constructive heuristics for different repair tasks.
	\item We connect evolution to inference through checked repair responses. A monotone-submodular objective retains complementary specialists, while a Repair Response Archive indexes the same region--heuristic evidence by local structure for later selection.
	\item We instantiate a checked repair interface for CVRP, TSP, VRPTW, and SDVRP. Under an exact problem-specific checker, sequential merge with rollback preserves complete-solution feasibility even when generated code fails.
\end{enumerate}
\section{Related Work}

\subsection{LLM Program Search for Heuristics}

LLM-based AHD searches over executable programs. FunSearch established LLM-guided mathematical-program search, and EoH jointly evolves design ideas and implementations~\cite{romera2024mathematical,liu2024evolution}. Later methods vary candidate development: ReEvo uses reflection, LLaMEA embeds generation in an evolutionary metaheuristic, and MCTS-AHD organizes refinement in a score-propagating tree~\cite{ye2024reevo,van2024llamea,Zheng2025}. All evaluate heuristic programs; SpecAHD uses this interface for two roles: upper programs expose repair tasks, and lower programs reconstruct them.

\subsection{Heuristic Sets, Archives, and Selection}

PartEvo clusters candidates by code or thought embeddings; HSEvo, MAP-Elites, and MEoH preserve behavioral or objective diversity~\cite{hu2025partevo,dat2025hsevo,mouret2015illuminating,kent2023bop,yao2025meoh}. These archives retain candidates, whereas SpecAHD archives observed heuristic--task responses. EoH-S selects a complementary set over complete-instance performance, and InstSpecHH maps instance features to heuristics~\cite{liu2026eohs,zhang2025instspechh}. SpecAHD instead specializes within an incumbent: its objective retains region-level complementarity, and a deterministic nonparametric router queries measured responses on structurally similar regions. Code distance and descriptor coverage receive no direct reward.

\subsection{Repair, Decomposition, and Large Neighborhood Search}

LNS, ruin-and-recreate, and adaptive LNS expose part of an incumbent solution and reconstruct it~\cite{shaw1998using,schrimpf2000record,ropke2006adaptive}. Recent LLM systems search parts of this interface. LLM-LNS evolves MILP neighborhood-selection strategies~\cite{llmlns2025}. G-LNS maintains destroy and repair populations and evaluates their pairwise synergy~\cite{zhao2026glns}. Full-component ALNS evolution searches seven modules, including operator selection and destroy-rate control, while MAP-Elites preserves distinct component behaviors~\cite{yu2026fullcomponent}. These methods learn a neighborhood policy, compatible operator populations, or a solver configuration. SpecAHD instead learns during evolution which executable repair program works for which local structure. Checked outcomes update this conditional relation as the upper-level search exposes new tasks and the lower-level search develops new repairs. The repertoire objective preserves specialists that improve tasks poorly handled by the current set, and the archive routes them when structurally similar regions appear again.

\section{SpecAHD Formulation}

\subsection{Problem Formulation}

Fix a routing problem class \(p\). Let \(x\in\mathcal X_p\) be an instance, \(\mathcal Y_p(x)\) its feasible solution space, and \(C_x:\mathcal Y_p(x)\rightarrow\mathbb R_{\geq0}\) the objective to be minimized. A baseline solver returns
\[
y_0=B(x)\in\mathcal Y_p(x).
\]
An exact checker \(F_p(x,y)\in\{0,1\}\) recognizes membership in \(\mathcal Y_p(x)\), with \(F_p(x,y_0)=1\). SpecAHD retains the complete solver and learns where to reopen \(y_0\) and how to reconstruct each opened part. We use \(\epsilon>0\) as a numerical safeguard.

The validator \(V_p\), descriptor \(\Phi_p\), reconstruction and merge operators \(R_p\) and \(M_p\), and checker \(F_p\) form the problem-specific interface. Given this interface, the alternating search, response-based repertoire objective, and archive protocol are shared across problem classes.

\subsection{Bilevel Repair Interface}
\label{sec:merge}

Given \(x\), a feasible incumbent \(y\in\mathcal Y_p(x)\), and a region budget \(K\), an upper program \(g\in\mathcal G_p\) proposes a raw region sequence. The validator returns at most \(K\) valid regions:
\[
\begin{aligned}
\widetilde{\mathbf z}&=g(x,y;K),\\
\mathbf z&=V_p(\widetilde{\mathbf z})=(z_1,\ldots,z_L),\qquad L\leq K.
\end{aligned}
\]
Each \(z_j=(I_j,b_j)\) contains a mutable support \(I_j\) and the boundary information \(b_j\) required for reconnection. The validator removes empty, duplicate, oversized, and locally invalid regions and enforces pairwise-disjoint supports: customers for CVRP and VRPTW, tour positions for TSP, and deliveries for SDVRP.

Let \(\mathcal H_p\) be the search space of lower programs. Each lower program \(h\in\mathcal H_p\) is an executable constructive heuristic. The lower-level search retains
\[
\begin{gathered}
H=\{h_1,\ldots,h_m\}\subseteq\mathcal H_p,\qquad m\leq q,\\
h_i:\sigma_\tau\mapsto a_\tau,
\end{gathered}
\]
where \(\sigma_\tau\) is the local construction state and \(a_\tau\) is the next node, customer, or delivery action. At step \(j\), let \(r_j=(x,y^{(j-1)},z_j)\). Executing \(h\) produces \(u_j=R_p(x,y^{(j-1)},z_j;h)\), or \(u_j=\bot\) if execution or parsing fails. The candidate merge and checked update are
\begin{equation}
\label{eq:checked_merge}
\begin{aligned}
\bar y_p(r_j,h) &= M_p(y^{(j-1)},z_j,u_j),\quad u_j\neq\bot;\\[4pt]
\Gamma_p(r_j,h) &=
\begin{cases}
\bar y_p(r_j,h), &
\substack{u_j\neq\bot\ \text{and}\\
F_p(x,\bar y_p(r_j,h))=1},\\
y^{(j-1)}, & \text{otherwise}.
\end{cases}
\end{aligned}
\end{equation}
Thus, any failed or infeasible merge leaves the incumbent unchanged. Unlike unit-test feedback for generated search components~\cite{cao2024autotos}, \(F_p\) checks the merged complete solution, including capacity and time-window effects beyond one support.

Given the validated order, a router selects \(h_{i_j}\in H\), and checked sequential repair updates
\[
\begin{aligned}
y^{(0)}&=y_0,\\
r_j&=(x,y^{(j-1)},z_j),\\
y^{(j)}&=\Gamma_p(r_j,h_{i_j}),\qquad j=1,\ldots,L.
\end{aligned}
\]
The final solution is \(\hat y=y^{(L)}\). For a single-region task \(r=(x,y,z_j)\), the checked response score of \(h\) is
\begin{equation}
\label{eq:repair_score}
s_{h,r}=
\frac{[C_x(y)-C_x(\Gamma_p(r,h))]_+}
{\max\{\epsilon,C_x(y)\}}\in[0,1].
\end{equation}

\begin{proposition}[Feasibility under checked sequential repair]
If \(y_0\in\mathcal Y_p(x)\) and \(F_p\) exactly recognizes \(\mathcal Y_p(x)\), checked sequential repair returns \(\hat y\in\mathcal Y_p(x)\) for any validated region sequence and lower-heuristic choices.
\end{proposition}

\begin{proof}
The baseline solution \(y^{(0)}\) is feasible. If \(y^{(j-1)}\) is feasible, Equation~\eqref{eq:checked_merge} either accepts a solution certified by \(F_p\) or retains \(y^{(j-1)}\). Induction gives \(y^{(j)}\in\mathcal Y_p(x)\) for all \(j=0,\ldots,L\).
\end{proof}

\begin{table}[tb]
\centering
\small
\begin{tabular}{@{}p{0.15\linewidth}p{0.40\linewidth}p{0.33\linewidth}@{}}
\toprule
Problem & Upper-program output & Lower-program action \\
\midrule
CVRP & mutable customer supports & next customer \\
TSP & mutable tour positions & next node \\
VRPTW & mutable customer supports & feasible next customer \\
SDVRP & mutable delivery supports & next delivery action \\
\bottomrule
\end{tabular}
\caption{Problem-specific instantiations of the bilevel repair interface.}
\label{tab:two_layer_interface}
\end{table}

\subsection{Response Archive and Repertoire Selection}
\label{sec:response_archive_objective}

Because no single heuristic repairs every exposed structure well, we retain task-level responses rather than one aggregate score per heuristic. For a finite instance set \(\mathcal D\), \(g\) induces
\[
\mathcal R_g(\mathcal D)=
\{(x,B(x),z):x\in\mathcal D,\ z\in V_p(g(x,B(x);K))\}.
\]
For \(r=(x,y,z)\), let \(\phi_r=\Phi_p(x,y,z)\) be its structural descriptor and \(s_{h,r}\) its checked response score. Evaluating every candidate \(h\in\mathcal P_g\subseteq\mathcal H_p\) on every task in \(\mathcal R_g(\mathcal D_{\mathrm{train}})\) yields
\[
\mathcal S_g=
\{(r,\phi_r,h,s_{h,r}):r\in\mathcal R_g(\mathcal D_{\mathrm{train}}),\ h\in\mathcal P_g\}.
\]
We retain \(r\) because distinct tasks may share a descriptor. Completed tables accumulate in the Repair Response Archive:
\[
\mathcal A^{t+1}=\mathcal A^t\cup\mathcal S_{g^t},\qquad \mathcal A^0=\varnothing.
\]
The archive may be sparse for a later repertoire because later heuristics were not evaluated on every archived task. Repertoire selection uses the current dense table \(\mathcal S_g\); repair-time routing queries \(\mathcal A^t\).

Let \(\mathcal R_g=\mathcal R_g(\mathcal D_{\mathrm{train}})\), \(n_g=|\mathcal R_g|>0\), and \(\bar s_g(h)=n_g^{-1}\sum_{r\in\mathcal R_g}s_{h,r}\). With \(\max_{h\in\varnothing}s_{h,r}=0\), we select \(H\subseteq\mathcal P_g\), \(|H|\leq q\), by maximizing
\begin{equation}
\label{eq:lower_objective}
\begin{aligned}
J_g(H)={}&\frac{\beta}{q}\sum_{h\in H}\bar s_g(h)\\
&+\frac{1-\beta}{n_g}\sum_{r\in\mathcal R_g}\max_{h\in H}s_{h,r},
\qquad |H|\leq q.
\end{aligned}
\end{equation}
The first term rewards individual average quality; the second rewards taskwise coverage by the best selected response. Thus \(\beta\in[0,1]\) trades off quality and complementarity. As in heuristic-set design~\cite{liu2026eohs}, a candidate improves the second term only where it raises the pointwise response envelope; code distance is not rewarded explicitly.

For \(h\notin H\), the marginal gain is
\[
\begin{aligned}
\delta_g(h\mid H)
&=J_g(H\cup\{h\})-J_g(H)\\
&=\frac{\beta}{q}\bar s_g(h)
+\frac{1-\beta}{n_g}\sum_{r\in\mathcal R_g}
\left[s_{h,r}-\max_{h'\in H}s_{h',r}\right]_+.
\end{aligned}
\]
The ideal budget-\(q\) repertoire for this fixed dense response table is
\[
H^\star(g)\in\arg\max_{H\subseteq\mathcal P_g,\,|H|\leq q}J_g(H).
\]

\begin{proposition}[Greedy repertoire selection]
For fixed \(g\), \(\mathcal P_g\), and \(n_g>0\), \(J_g\) is nonnegative, normalized, monotone, and submodular. Greedy selection by \(\delta_g\) returns \(H^{\mathrm{gr}}\) satisfying
\[
J_g(H^{\mathrm{gr}})\geq(1-1/e)J_g(H^\star(g)).
\]
The selection problem is NP-hard when \(\beta=0\) and repair scores are binary.
\end{proposition}

\begin{proof}
The first term is a nonnegative modular function. For each task, the second takes the largest selected response; adding a heuristic cannot decrease this maximum, and its marginal gain shrinks as the selected set grows. Since \(J_g(\varnothing)=0\), their nonnegative sum is normalized, monotone, and submodular. The cardinality-constrained greedy bound follows from Nemhauser, Wolsey, and Fisher~\cite{nemhauser1978analysis}. With \(\beta=0\) and binary scores, the objective reduces to Maximum Coverage, which is NP-hard~\cite{Garey1979}.
\end{proof}

This guarantee holds only for fixed \(g\), its induced task set, and the evaluated candidate pool; it does not imply convergence or global optimality of the alternating bilevel search.

\subsection{Coupled Design Objective}

The upper program determines \(\mathcal R_g\) and hence the dense response table used by the lower-level search; the selected repertoire in turn determines the repair quality used to evaluate \(g\). For the ideal objective, let \(\rho_{g,H}\) be a router fitted to \(\mathcal S_g\) and restricted to \(H\), and let \(\hat y(x;g,H,\rho_{g,H})\) be the complete solution returned by checked sequential repair. Its improvement is
\[
U_x(g,H)=
\frac{C_x(B(x))-C_x(\hat y(x;g,H,\rho_{g,H}))}
{\max\{\epsilon,C_x(B(x))\}}.
\]
Let \(N_{\mathrm{tr}}=|\mathcal D_{\mathrm{train}}|\). With \(\Omega(g,x)\geq0\) penalizing invalid proposals and rejected merges and \(\lambda\geq0\) controlling its weight, the ideal coupled objective is
\[
g^\star\in\arg\max_{g\in\mathcal G_p}
\frac{1}{N_{\mathrm{tr}}}\sum_{x\in\mathcal D_{\mathrm{train}}}
\left[U_x(g,H^\star(g))-\lambda\Omega(g,x)\right].
\]
A fixed partition would remove this coupling from the search. Changing \(g\) changes the lower-level training tasks, while changing \(H\) changes the measured value of \(g\). Exact optimization would require a fresh lower-level search and dense response table for every upper candidate. We therefore alternate the two searches: each round updates the repertoire and archive for the current upper program, then freezes this repair system while evaluating the next upper candidates. The implementation replaces \(\rho_{g,H}\) with the archive-based router \(\rho_{\mathcal A,H}\).

\section{Proposed Method}

\subsection{Framework Overview}

SpecAHD alternates between region exposure and repair search. An upper program induces bounded repair tasks; their dense response table selects a complementary repertoire, while the persistent archive supports region-conditioned routing. One training round follows
\(g^t\!\rightarrow\!\mathcal R_{g^t}\!\rightarrow\!\mathcal S^t\!\rightarrow\!(H^{t+1},\mathcal A^{t+1})\!\rightarrow\!g^{t+1}\).
The selected repertoire and updated archive are then frozen while checked complete-solution evaluations determine \(g^{t+1}\). Figure~\ref{fig:framework} summarizes training and frozen test-time execution.

\begin{figure*}[tb]
\centering
\includegraphics[width=\textwidth]{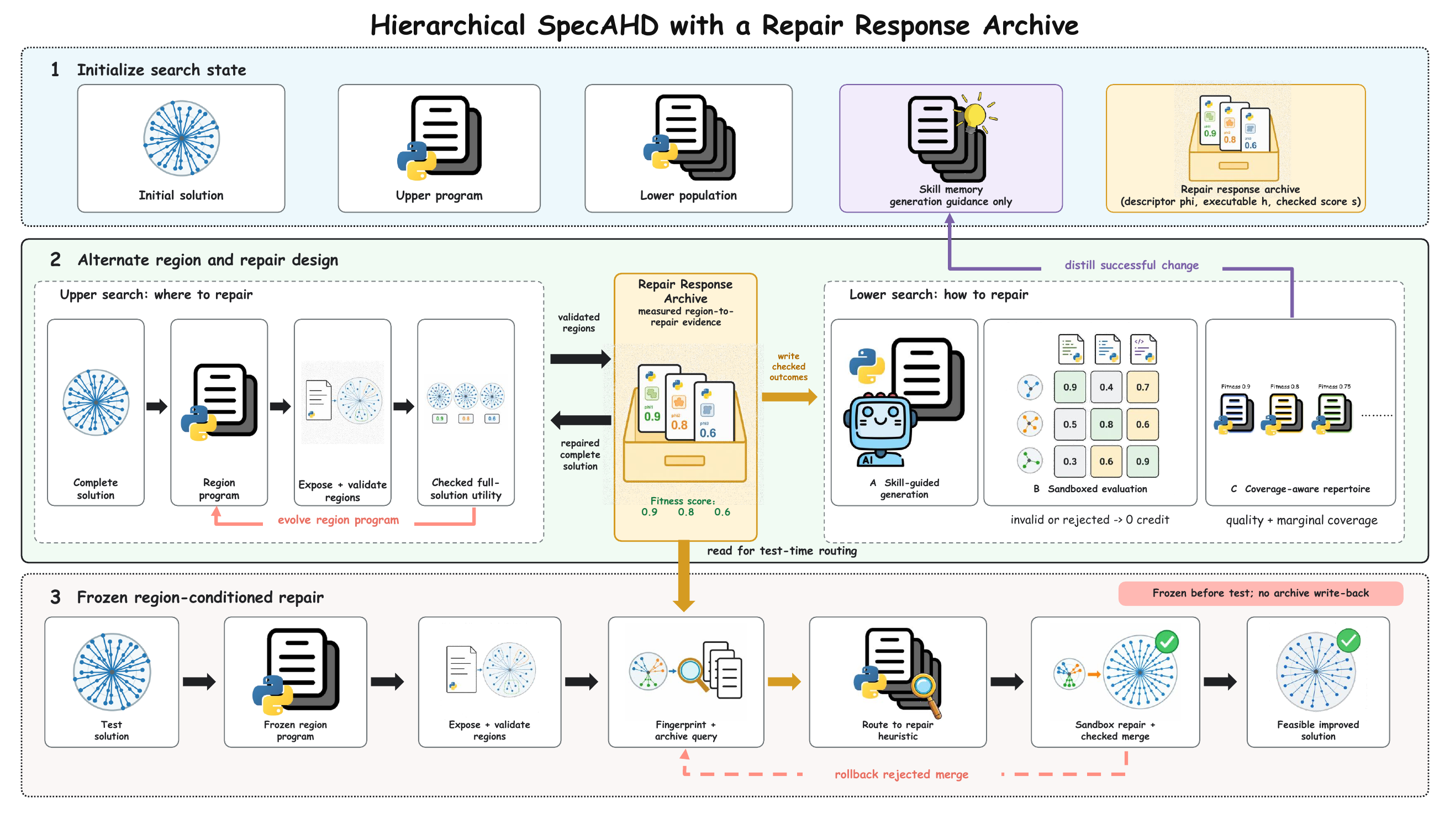}
\caption{Overview of SpecAHD. Training alternates upper-level region design and lower-level repair search. Checked task-level responses select a repertoire and update the archive. At test time all components are frozen; archive routing selects a retained heuristic for each exposed region, and checked merge rolls back failed or infeasible repairs.}
\label{fig:framework}
\end{figure*}

\subsection{Learning Where to Repair}

For an instance \(x\) and feasible incumbent \(y\), an executable upper program returns an ordered sequence of raw region proposals, \(\widetilde{\mathbf z}=g(x,y;K)\). The problem-specific validator removes proposals that violate the region interface and returns \(\mathbf z=V_p(\widetilde{\mathbf z})=(z_1,\ldots,z_L)\) before any repair program is called. Because \(g\) observes both the instance and the incumbent, the exposed regions can adapt to the instance structure and solution state rather than follow a fixed geometric partition.

Across the training set, these validated regions define the repair-task set \(\mathcal R_g\). The upper-level search decides where to repair, while the lower-level search reconstructs each selected region. Each lower candidate therefore receives an explicit repair task and only needs to choose construction actions within the supplied region.

Under the coupled objective, an upper program is rewarded only when the retained specialists turn its exposed regions into a better complete solution. To attribute this improvement to region selection, all upper candidates in round \(t\) use the same frozen \((H^{t+1},\mathcal A^{t+1})\): they may change the proposed regions, but not the repair system that scores them. Starting from \(y_0=B(x)\), let \(\hat y_g^t(x)\) denote the checked output of this frozen pipeline. The resulting upper utility is

\[
U_x^t(g)
=
\frac{C_x(y_0)-C_x(\hat y_g^t(x))}
{\max\{\epsilon,C_x(y_0)\}}.
\]
For \(N_{\mathrm{tr}}=|\mathcal D_{\mathrm{train}}|\), the upper-level search maximizes the finite-sample objective
\[
\begin{aligned}
\widehat J_U^t(g)=\frac{1}{N_{\mathrm{tr}}}
\sum_{x\in\mathcal D_{\mathrm{train}}}\big[&U_x^t(g)\\
&-\lambda\Omega(g,x)\big],
\end{aligned}
\]
For \(\widetilde{\mathbf z}=g(x,y_0;K)\) and \(\mathbf z=V_p(\widetilde{\mathbf z})\), define \(L=|\mathbf z|\) and let \(N_{\mathrm{rb}}\) be the number of rejected merges. With \(\alpha_1,\alpha_2,\alpha_3\geq0\), the penalty is
\[
\begin{aligned}
\Omega(g,x)={}&\alpha_1\mathbb{I}[L=0]\\
&+\alpha_2\frac{|\widetilde{\mathbf z}|-L}{\max\{1,|\widetilde{\mathbf z}|\}}\\
&+\alpha_3\frac{N_{\mathrm{rb}}}{\max\{1,L\}}.
\end{aligned}
\]
These terms penalize empty validated output, validator rejection rate, and rollback rate. Region size, overlap, and local structural constraints remain hard constraints enforced by \(V_p\).

Both searches use the EoH mutation and recombination interface; Appendix~G gives the generation contracts and failure handling~\cite{liu2024evolution}.

\subsection{Learning How to Repair}

Every executable lower candidate is evaluated on every repair task in the current round, yielding the dense response table \(\mathcal S^t\). A single average score would obscure task-specific strengths, so repertoire selection uses taskwise responses to preserve complementary behavior~\cite{liu2026eohs}. Candidates that fail to parse or launch in the sandbox are discarded; timeouts, invalid task outputs, and rejected merges receive a zero response score. Starting from \(H=\varnothing\), greedy selection repeatedly adds the candidate with the largest positive \(\delta_{g^t}(h\mid H)\) until no positive gain remains or \(|H|=q\). A candidate enters the repertoire only if it improves the measured set objective; code distance alone provides no reward.

Lower candidates follow the EoH parent-rewriting process~\cite{liu2024evolution}. Retained records summarize code changes, successful task conditions, and failures for later mutations~\cite{wu2025hercules}; they guide generation only, while checked responses determine repertoire selection and routing. Appendix~G gives the full prompt and record protocol.

\subsection{Conditioning Repair on Measured Responses}

The repertoire is useful only if the router can select the relevant specialist at each repair call. Applying one globally strongest heuristic to every region would ignore the response differences used to select \(H\). SpecAHD stores each checked record \((r,\phi_r,h,s_{h,r})\) in \(\mathcal A\). When a new region is exposed, the router queries these records. The same read-only query is used to score upper candidates with a frozen repair system and to repair test instances after search. Since records are collected under different upper programs and candidate pools, the archive can be sparse for a later repertoire. For a new region, the router estimates each heuristic's response from structurally similar archived tasks and excludes heuristic-task pairs that were never evaluated.

Before repairing \(z_j\), we compute its descriptor from the current feasible incumbent,
\[
\phi_j=\Phi_p(x,y^{(j-1)},z_j).
\]

Appendix~B lists each problem-specific descriptor together with its repair interface.

The descriptors are problem-specific and used only for archive lookup. The repertoire objective \(J_g\) remains based on checked responses. We standardize each coordinate with training statistics, group records by task identity, and define query-to-task similarity as
\[
\kappa_{j\ell}=
\max\!\left\{0,
\frac{\widetilde\phi_j^\top\widetilde\phi_\ell}
{\max\{\epsilon,\|\widetilde\phi_j\|_2\|\widetilde\phi_\ell\|_2\}}
\right\}.
\]
Let \(\mathcal N_k(\phi_j)\) contain the \(k\) nearest archived tasks, truncated to archive size. We set \(w_\ell=\kappa_{j\ell}\). If all similarities are zero, we instead use uniform weights. With observed-pair indicator \(o_{i\ell}\), define
\begin{equation}
\label{eq:response_router}
\begin{aligned}
d_i(\phi_j)&=\sum_{\ell\in\mathcal N_k(\phi_j)}o_{i\ell}w_\ell,\\
Q_i(\phi_j)&=
\begin{cases}
\displaystyle
\frac{\sum_{\ell\in\mathcal N_k(\phi_j)}o_{i\ell}w_\ell s_{i\ell}}
{d_i(\phi_j)}, & d_i(\phi_j)>0,\\[2mm]
\bar s_{\mathcal A}(h_i), & d_i(\phi_j)=0.
\end{cases}
\end{aligned}
\end{equation}
Here \(\bar s_{\mathcal A}(h_i)\) is the mean archived score of \(h_i\), with value zero when \(h_i\) has no archived records. The router selects the heuristic with the largest \(Q_i(\phi_j)\), breaking ties first by archive mean and then by fixed index. This deterministic weighted-neighbor rule requires no trained selector~\cite{cover1967nearest,alissa2023algorithm,zhang2025instspechh}. The selected heuristic is executed through \(R_p\), and the checked merge then accepts or rejects its repair.

\subsection{Alternating the Two Searches under Finite Budgets}

A fresh lower-level search for every upper candidate would be too costly. SpecAHD therefore alternates finite searches. At round \(t\), \(g^t\) exposes \(n_t=|\mathcal R_{g^t}|\) tasks. When \(n_t=0\), the previous lower-level search state is retained and \(g^t\) incurs the empty-region penalty.

For \(n_t>0\), every lower candidate retained in \(\mathcal P^{t+1}\) has a checked response on each of the \(n_t\) tasks, giving \(n_t|\mathcal P^{t+1}|\leq B_L\). After at most \(T_L\) generations,
\[
(\mathcal P^{t+1},\mathcal S^t)
=\operatorname{LowerSearch}(g^t,\mathcal P^t;T_L,B_L).
\]
We then update \(\mathcal A^{t+1}=\mathcal A^t\cup\mathcal S^t\) and select
\[
H^{t+1}=\operatorname{Greedy}_q(\mathcal P^{t+1};J_{g^t}).
\]
The upper-level search freezes \((\mathcal A^{t+1},H^{t+1})\) and evaluates every completed candidate on all \(N_{\mathrm{tr}}\) training instances, subject to at most \(T_U\) generations and total budget \(B_U\):
\[
g^{t+1}
=\operatorname{UpperSearch}(g^t,H^{t+1},\mathcal A^{t+1};T_U,B_U,K).
\]
The procedure stops at the round limit or shared budget. These budgeted updates approximate the nested coupled objective without assuming that either program search reaches its optimum.

\begin{table*}[!t]
\centering
\small
\begin{tabular}{@{}lrrrrrrrrrrr@{}}
\toprule
Method & \multicolumn{6}{c}{CVRP} & \multicolumn{5}{c}{TSP} \\
 & 100 & 200 & 300 & 500 & 800 & $t$(s) & 200 & 500 & 800 & 1000 & $t$(s) \\
\midrule
RS & 14.02k & 27.24k & 42.63k & 82.45k & 135.81k & 0.08 & 58.21 & 145.85 & 199.87 & 300.92 & 0.01 \\
HC & 15.27k & 28.35k & 44.22k & 86.45k & 136.80k & 0.18 & 82.69 & 211.16 & 292.06 & 439.91 & 0.01 \\
FS & 14.07k & 27.48k & 42.52k & 82.50k & 135.97k & 0.07 & 82.69 & 211.16 & 292.06 & 439.91 & 0.01 \\
EoH & 13.61k & 26.42k & 42.84k & 82.57k & 135.04k & 0.09 & 8.81 & 14.33 & 14.86 & 21.77 & 0.50 \\
ReEvo & 16.79k & 29.85k & 47.10k & 86.60k & 138.12k & 0.79 & 8.90 & 14.55 & 14.96 & 22.30 & 10.58 \\
MLES & 18.63k & 36.90k & 56.59k & 108.12k & 161.63k & 3.10 & 33.67 & 80.31 & 107.51 & 161.81 & 0.30 \\
PE & 19.59k & 35.38k & 59.07k & 102.32k & 172.91k & 0.48 & 9.18 & 14.95 & 15.29 & 22.68 & 0.11 \\
LHNS & 26.24k & 57.49k & 81.16k & 163.39k & 250.52k & 0.06 & 19.01 & 38.58 & 36.74 & 88.04 & 1.58 \\
MCTS & 13.87k & 26.83k & 42.83k & 81.84k & 135.02k & 0.12 & 8.96 & 14.77 & 15.33 & 22.54 & 7.26 \\
EoH-S & 57.98k & 44.71k & 62.46k & 75.72k & 126.85k & 0.39 & 10.90 & 14.36 & 14.69 & 21.54 & 11.12 \\
\midrule
SpecAHD-G & 12.96k & \textbf{23.40k} & \textbf{39.94k} & 75.54k & 126.25k & 0.06 & 8.61 & 13.42 & 14.52 & 21.49 & 0.39 \\
SpecAHD-D & \textbf{12.91k} & 23.59k & 40.12k & \textbf{49.98k} & \textbf{75.91k} & 0.06 & \textbf{8.12} & \textbf{12.06} & \textbf{13.99} & \textbf{20.66} & 0.47 \\
\bottomrule
\end{tabular}
\caption{Held-out CVRP and TSP validation by problem size. All methods use three independent runs; costs aggregate completed runs. Size columns give node or customer counts, $t$ mean end-to-end seconds per instance, and k thousands. Lower is better. RS/HC/FS/PE/MCTS denote RandSample/HillClimb/FunSearch/PartEvo/MCTS-AHD; G/D denote the GLM/DeepSeek SpecAHD variants.}
\label{tab:validation-all}
\end{table*}

\begin{table*}[!t]
\centering
\small
\begin{tabular}{@{}lrrrrrrrrrrrr@{}}
\toprule
Method & \multicolumn{6}{c}{VRPTW} & \multicolumn{6}{c}{SDVRP} \\
 & 200 & 400 & 600 & 800 & 1000 & $t$(s) & 200 & 250 & 300 & 400 & 500 & $t$(s) \\
\midrule
RS & 47.25k & 85.06k & 129.38k & 163.09k & 196.19k & 0.19 & 7.41k & 11.51k & 14.32k & 15.83k & 23.71k & 0.01 \\
HC & 84.17k & 162.44k & 258.70k & 348.41k & 432.48k & 0.10 & 7.34k & 11.44k & 14.26k & 15.79k & 23.66k & 0.02 \\
FS & 79.68k & 151.66k & 239.31k & 318.60k & 393.69k & 0.20 & 13.40k & 19.43k & 23.13k & 28.79k & 39.48k & 0.02 \\
EoH & 46.41k & 84.23k & 127.82k & 161.61k & 197.29k & 0.41 & 7.40k & 11.53k & 14.33k & 15.82k & 23.69k & 0.05 \\
ReEvo & 55.83k & 100.60k & 150.73k & 193.65k & 235.22k & 12.95 & 7.44k & 12.33k & 15.73k & 16.81k & 26.23k & 2.26 \\
MLES & 39.92k & 76.62k & 132.84k & 174.68k & 221.12k & 0.34 & 7.36k & 11.44k & 14.27k & 15.81k & 23.69k & 2.10 \\
PE & 47.42k & 83.25k & 124.93k & 158.40k & 189.05k & 0.45 & 7.41k & 11.52k & 14.34k & 15.84k & 23.70k & 8.05 \\
LHNS & 68.30k & 123.18k & 185.12k & 243.43k & 296.60k & 0.20 & 7.41k & 11.53k & 14.34k & 15.85k & 23.74k & 0.11 \\
MCTS & 41.64k & 73.71k & 107.49k & 136.51k & 165.05k & 0.63 & 7.34k & 11.44k & 14.26k & 15.78k & 23.65k & 0.15 \\
EoH-S & 48.30k & 66.09k & 94.74k & 119.30k & 143.78k & 1.43 & 12.08k & 13.80k & 11.36k & 21.87k & 15.37k & 2.55 \\
\midrule
SpecAHD-G & \textbf{33.45k} & 90.81k & \textbf{88.28k} & 181.23k & 226.09k & 0.21 & \textbf{3.31k} & 12.17k & 15.18k & 16.71k & 25.81k & 1.46 \\
SpecAHD-D & 46.48k & \textbf{61.03k} & 142.69k & \textbf{108.32k} & \textbf{121.49k} & 1.25 & 7.43k & \textbf{4.84k} & \textbf{5.92k} & \textbf{6.97k} & \textbf{9.96k} & 0.04 \\
\bottomrule
\end{tabular}
\caption{Held-out VRPTW and SDVRP validation by problem size. All methods use three independent runs; costs aggregate completed runs. Size columns give customer counts, $t$ mean end-to-end seconds per instance, and k thousands. Lower is better.}
\label{tab:validation-vrptw-sdvrp}
\end{table*}

\begin{table}[tb]
\centering
\small
\begin{tabular}{@{}lrrrr@{}}
\toprule
& \multicolumn{2}{c}{Lower-level search} & \multicolumn{2}{c}{Frozen repertoire} \\
\cmidrule(lr){2-3}\cmidrule(l){4-5}
Variant & TSP $\uparrow$ & CVRP $\downarrow$ & TSP $\uparrow$ & CVRP $\uparrow$ \\
& (\%) & (k) & (\%) & (\%) \\
\midrule
Full & \textbf{$5.658$} & \textbf{$59.943$} & \textbf{$-0.337$} & \textbf{$0.578$} \\
Single embedding & $-0.015$ & $83.512$ & $-1.857$ & $-0.350$ \\
w/o semantic filter & $3.783$ & $60.712$ & $-0.986$ & $-1.304$ \\
w/o skill guidance & $4.227$ & $118.517$ & $-1.141$ & $-0.079$ \\
\bottomrule
\end{tabular}
\caption{Held-out ablations of the lower-level search and the full system with a frozen repertoire. Each setting uses three independent runs; cells average completed runs. Appendix~F gives run-level values, sample standard deviations, and completion status.}
\label{tab:specahd-validation-lower-ablation}
\label{tab:specahd-validation-upper-ablation}
\end{table}

\section{Experiments}

We evaluate SpecAHD on four routing problems, compare it with ten program-search baselines, and test both component contributions and region-level specialization.

\subsection{Experimental Setup}

\paragraph{Tasks and data splits.}
Following EoH-S, the lower-level searches for TSP and CVRP retain 128 clustered Euclidean instances (10--200 cities) and 256 uniform instances (20--200 nodes; capacities 10--150), respectively~\cite{liu2026eohs}. Lower-level training data for VRPTW and SDVRP follow the generation rules of their respective 60-instance large-instance upper-level training pools. Upper-level training uses independently generated large instances. The synthetic CVRP pool contains 60 instances and follows the uniform-instance rules used for lower-level training. The VRPTW pool follows Solomon/Homberger C/R/RC spatial families and type-1/type-2 time-window regimes at 200, 400, 600, 800, and 1,000 customers~\cite{solomon1987vrptw,homberger1999two}. The SDVRP pool samples 200--500 customers from uniform, clustered, ring, and TSPLIB-like layouts under medium, heavy, and mixed demand; PASA-style splitting and PyVRP construct its baseline solutions~\cite{torkzaban2023pasa,wouda2024pyvrp}. A separately seeded 60-instance set supports checkpoint selection. Final tests use CVRP-XL, TSPLIB above 1,000 nodes, Solomon/Homberger VRPTW, and SET-1--4 SDVRP~\cite{queiroga2026xl,reinelt1991tsplib,chen2007split,belenguer2000lower,archetti2008optimization}; Appendix~D lists every result.

\paragraph{Baselines.}
The ten baselines span memory-free sampling, $(1+1)$ search, populations, partitioned populations, neighborhood search, and tree search: RandSample, HillClimb~\cite{zhang2024importance}, FunSearch~\cite{romera2024mathematical}, EoH~\cite{liu2024evolution}, ReEvo~\cite{ye2024reevo}, MLES~\cite{hu2026mles}, PartEvo~\cite{hu2025partevo}, LHNS~\cite{xie2025lhns}, MCTS-AHD~\cite{Zheng2025}, and EoH-S~\cite{liu2026eohs}. All use the same LLM4AD task and sandbox interfaces~\cite{liu2024llm4ad}. This controlled interface addresses comparison and feasibility risks emphasized by recent AHD benchmarks~\cite{zhang2024importance,sun2026cobench}; Appendix~C records the routing-specific adapters.

\paragraph{LLM backbones, search budget, and reporting.}
We use GLM-4-Flash and DeepSeek-V4-Flash and report both SpecAHD variants. Within each backbone, all methods share the task, template, training data, and a 1,000-program budget. For SpecAHD, this budget covers both the upper-level and lower-level searches. We run each method independently three times. Each run uses four sampler workers, four evaluator workers, and a population of ten. We average completed runs while retaining partial runs, parser failures, infeasible outputs, and timeouts in the record. Appendices~C and F give the full protocol and run-level ablation results.

Across the 19 validation size groups, one of the two SpecAHD backbones obtains the lowest cost in every group (Table~\ref{tab:validation-all}). We report both because the stronger backbone varies by problem and size.

Held-out validation measures size scaling under independently seeded synthetic distributions, whereas the named public sets test transfer to established benchmarks. The public comparison uses the lowest numeric baseline on each row; Table~\ref{tab:public-summary} summarizes the outcomes, and Appendix~D gives the underlying costs.

\begin{table}[tb]
\centering
\small
\begin{tabular}{@{}lrr@{}}
\toprule
Problem & Reported & Win \\
\midrule
CVRP & 79 & 71 \\
TSP & 30 & 28 \\
VRPTW & 331 & 315 \\
SDVRP & 95 & 68 \\
\bottomrule
\end{tabular}
\caption{Public benchmark summary against the per-instance baseline envelope. Appendix~D reports every objective value.}
\label{tab:public-summary}
\end{table}

\subsection{Ablation Summary}

We ablate the three-embedding memory, semantic duplicate filter, and skill guidance on 60 held-out TSP and CVRP instances, using three independent runs per setting. The ablation of the lower-level search evolves specialists; the full-system study changes the upper-level search configuration with a frozen lower repertoire. Full SpecAHD has the best mean in all four settings (Table~\ref{tab:specahd-validation-lower-ablation}). One single-embedding lower-CVRP run and one no-filter full-system CVRP run time out. Appendix~F reports completion counts and runtimes; per-run values are available in the released artifacts.

\subsection{Repair Specialization Validation}

We replay all ten retained programs on 142 held-out TSP and 430 held-out CVRP regions. We fit the $k=5$ router only on training regions and compare it with the global training-best heuristic, uniform random selection over 30 seeded repeats, and a hindsight oracle.

\begin{table}[tb]
\centering
\small
\begin{tabular}{@{}llrr@{}}
\toprule
Problem & Policy & Response $\uparrow$ & $\Delta$GB (\%) $\uparrow$ \\
\midrule
TSP & SpecAHD & 0.8695 & $+0.659$ \\
    & Global best & 0.8200 & 0.000 \\
    & Random & $0.7551{\pm}0.0248$ & $-1.371{\pm}0.502$ \\
    & Oracle & 1.0000 & $+2.552$ \\
\midrule
CVRP & SpecAHD & 0.9493 & $+0.188$ \\
     & Global best & 0.9489 & 0.000 \\
     & Random & $0.8038{\pm}0.0114$ & $-12.027{\pm}1.989$ \\
     & Oracle & 1.0000 & $+1.998$ \\
\bottomrule
\end{tabular}
\caption{Routing on held-out regions. Uniform random selection reports mean and standard deviation over 30 seeded repeats; the other policies are deterministic. Response is normalized within each region, and $\Delta$GB is signed improvement over the global training-best heuristic.}
\label{tab:response-routing}
\end{table}

The router improves TSP cost over the global best by $0.659\%$. On CVRP, its normalized response is higher ($0.9493$ vs.\ $0.9489$), and its signed cost improvement is $+0.188\%$. The gain is smaller than on TSP but remains positive. The oracle gains of $2.552\%$ and $1.998\%$ show that further improvements are attainable on both tasks. Appendix~E reports regret, oracle-hit rate, repertoire growth, and representative programs.

\section{Conclusion and Limitations}

In this work, we introduced SpecAHD, a coupled bilevel framework that brings specialization inside large-scale routing instances. SpecAHD learns to expose bounded repair regions and evolves a complementary repertoire whose members are retained through checked task-level responses. For a fixed induced task set and candidate pool, we proved that the repertoire objective is monotone submodular, which gives greedy selection a \((1-1/e)\) approximation guarantee. The same responses populate a Repair Response Archive for region-conditioned routing, while checked sequential repair preserves complete-solution feasibility. Across four routing problems, multiple LLM backbones, and large-scale benchmarks, SpecAHD reduces objective cost by up to \(57.7\%\) against the strongest competing AHD baseline and generalizes beyond its training distributions.

SpecAHD shifts LLM-driven AHD from selecting one broadly effective program to choosing among specialists according to local solution structure. The theoretical guarantee does not establish convergence of the alternating search, and archive routing can weaken when its neighborhoods are sparse or structurally shifted. Future work should reduce the dependence on hand-designed descriptors and make routing more reliable under distribution shift. It should also account for interactions among successive repairs and lower generated-code execution costs.

\bibliography{references}

\end{document}


\maketitle

\section{A. Additional Proofs}

\subsection{Feasibility under Checked Merge}

We restate the feasibility claim using the notation of the main paper. Let
\(y^{(0)}=y_0\) be feasible. At repair step \(j\), the adapter produces a
candidate merge \(\bar y_p(r_j,h_{i_j})\), and the checked operator returns
\[
\Gamma_p(r_j,h_{i_j})=
\begin{cases}
\bar y_p(r_j,h_{i_j}),
&F_p(x,\bar y_p(r_j,h_{i_j}))=1,\\
y^{(j-1)},&\text{otherwise}.
\end{cases}
\]

\begin{proposition}
If \(F_p\) exactly recognizes the feasible set \(\mathcal Y_p(x)\), then
checked sequential repair returns a feasible complete solution for any
validated region sequence and any collection of lower heuristics.
\end{proposition}

\begin{proof}
We use induction over the repair sequence. The base case holds because
\(y^{(0)}=y_0\in\mathcal Y_p(x)\). Assume
\(y^{(j-1)}\in\mathcal Y_p(x)\). If the candidate merge passes the checker,
exactness of \(F_p\) gives
\(\bar y_p(r_j,h_{i_j})\in\mathcal Y_p(x)\), and this candidate becomes
\(y^{(j)}\). If it fails, the operator returns \(y^{(j-1)}\), which is
feasible by the induction hypothesis. Hence \(y^{(j)}\) is feasible in
both cases. Applying the argument through step \(L\) gives
\(\hat y=y^{(L)}\in\mathcal Y_p(x)\).
\end{proof}

The argument does not require the merge order to be exchangeable. Different
orders may lead to different feasible solutions because an accepted repair
changes the state seen by later repairs.

\subsection{Greedy Repertoire Selection}

Fix an upper program \(g\), its finite repair-task set \(\mathcal R_g\), and
a finite lower candidate pool \(\mathcal P_g\). For
\(H\subseteq\mathcal P_g\), define
\[
J_g(H)=
\frac{\beta}{q}\sum_{h\in H}\bar s_g(h)
+\frac{1-\beta}{n_g}\sum_{r\in\mathcal R_g}
\max_{h\in H}s_{h,r},
\]
where \(n_g=|\mathcal R_g|>0\), all scores are nonnegative, and the maximum
over the empty set is zero.

\begin{proposition}
The function \(J_g\) is normalized, nonnegative, monotone, and submodular.
Under \(|H|\leq q\), greedy selection attains a \((1-1/e)\) approximation.
Maximizing \(J_g\) is NP-hard even when \(\beta=0\) and all repair scores
are binary.
\end{proposition}

\begin{proof}
The first term is a nonnegative modular function. It is therefore normalized,
monotone, and submodular. Consider one task \(r\) in the second term and set
\(f_r(H)=\max_{h\in H}s_{h,r}\). Its marginal gain for a candidate
\(u\notin H\) is
\[
\Delta_r(u\mid H)=
\left[s_{u,r}-\max_{h\in H}s_{h,r}\right]_+.
\]
For \(A\subseteq B\), we have
\(\max_{h\in A}s_{h,r}\leq\max_{h\in B}s_{h,r}\), and thus
\(\Delta_r(u\mid A)\geq\Delta_r(u\mid B)\). This is the diminishing-returns
condition. Each \(f_r\) is also normalized, nonnegative, and monotone.
Nonnegative weighted sums preserve these properties, so \(J_g\) has all four.
The standard result for cardinality-constrained maximization of a normalized
monotone submodular function gives the \((1-1/e)\) greedy bound
\cite{nemhauser1978analysis}.

For hardness, take an instance of Maximum Coverage with universe
\(\mathcal U\), subsets \(S_h\subseteq\mathcal U\), and a budget \(q\).
Create one repair task for every element \(r\in\mathcal U\), one candidate
heuristic for every subset, set \(\beta=0\), and define
\(s_{h,r}=1\) exactly when \(r\in S_h\). Then
\(n_gJ_g(H)\) is the number of covered elements. An optimizer for the
repertoire problem would therefore solve Maximum Coverage~\cite{Garey1979}.
\end{proof}

The guarantee concerns selection from an already evaluated candidate pool.
It does not cover program generation, archive routing, alternation between the
two searches, or the quality of the exposed repair tasks.

\section{B. Problem-Specific Repair Interfaces}

SpecAHD shares its search, dense response table, repertoire selection, and routing
rule across problems. The exposed object and the feasibility boundary remain
problem specific. Table~\ref{tab:supp-repair-interfaces} states the interface
used for each task. These choices follow the usual distinction in LNS between
selecting a neighborhood and reconstructing the released decisions
\cite{shaw1998using,ropke2006adaptive}. Related large-routing methods also
delegate spatially local subproblems to a separate solver or learn the repair
operator within a fixed neighborhood interface
\cite{li2021delegate,hottung2020neural}.

\begin{table*}[tb]
\centering
\small
\begin{tabular}{@{}p{0.10\textwidth}p{0.20\textwidth}p{0.25\textwidth}p{0.35\textwidth}@{}}
\toprule
Problem & Region returned by the upper program & Lower construction state & Validation and complete-solution check \\
\midrule
TSP & A route-order segment with fixed endpoints & Current node, destination endpoint, unvisited interior nodes, and the distance matrix & Keep segments within the size and count budgets; remove repeated or overlapping tour positions; verify that the merged tour contains every node once and returns to its start. \\
CVRP & A set of mutable customers, excluding the depot & Current node, remaining vehicle capacity, feasible unvisited customers, demands, and distances & Remove depot indices, duplicates, and oversized groups; after reconstruction, check customer coverage, depot boundaries, capacity, and total route cost. \\
VRPTW & A customer set whose assignments may be rebuilt & CVRP state together with current time, ready times, due times, and service times & Apply the customer-group checks, then verify capacity and every arrival, waiting, service, and due-time constraint in the reconstructed complete solution~\cite{solomon1987vrptw,homberger1999two}. \\
SDVRP & A customer-delivery support that may span several routes & Residual demands, remaining capacity, current route state, and the customers still requiring delivery & Preserve delivered quantity rather than visit uniqueness; verify nonnegative deliveries, route capacity, demand conservation, depot boundaries, and complete service~\cite{chen2007split,belenguer2000lower}. \\
\bottomrule
\end{tabular}
\caption{Problem-specific repair interfaces. The upper program identifies mutable support; the lower program chooses the next constructive action inside that support.}
\label{tab:supp-repair-interfaces}
\end{table*}

\paragraph{Region validation.}
The upper output is treated as an untrusted proposal. Each adapter converts
indices to integers, removes out-of-range entries, enforces the maximum number
and size of regions, and removes repeated mutable decisions. For TSP, retained
nodes must occupy a contiguous interval in the current route. For the three
vehicle-routing tasks, the depot is never included in a customer group. A raw
proposal that becomes empty after these checks receives the empty-region
penalty and invokes no lower program.

\paragraph{Bounded constructive calls.}
The framework presents only feasible next actions to a lower heuristic. If a
returned action is unavailable or violates the local state, a deterministic
fallback selects from the same feasible set. TSP construction keeps the two
segment endpoints fixed. CVRP and SDVRP start a new depot route when no
remaining customer fits the residual capacity. VRPTW filters next customers by
capacity and reachable due time before calling the generated program. Parsing,
execution, and timeout failures are recorded as rejected task outcomes.

\paragraph{Complete-solution boundary.}
Local filtering does not replace the checker. Capacity, time, and delivery
constraints can depend on decisions outside one repair call, so the merged
complete solution is checked before it replaces the incumbent. The feasibility
claim in Section~A applies to this final accept-or-rollback boundary, not to the
generated program in isolation.

\section{C. Reproducibility Protocol}

\paragraph{Data separation.}
The lower-level search, upper-level search, model selection, and final
testing use disjoint files. The TSP lower-level search uses four 32-instance synthetic
sets, giving 128 training instances. The CVRP lower-level search uses 256 synthetic
instances. VRPTW and SDVRP lower-search instances use the same generation
rules as their respective large-instance upper-level training pools. The upper-level searches for TSP and CVRP
each use 60 independently generated synthetic instances; the CVRP pool
follows the same uniform-instance rules as its lower-search data. Their held-out
ablation sets contain another 60 instances generated from different seeds and
spatial configurations. Public instances are reserved for the final test. They
are not used to generate skills, normalize structural descriptors, populate
repair responses, or choose a checkpoint.

\begin{table}[tb]
\centering
\small
\begin{tabular}{@{}p{0.16\columnwidth}p{0.54\columnwidth}p{0.16\columnwidth}@{}}
\toprule
Problem & Held-out/public test source & Instances \\
\midrule
TSP & TSPLIB, 1,001--20,000 nodes~\cite{reinelt1991tsplib} & 31 \\
CVRP & CVRP-XL, 1,000--10,000 customers~\cite{queiroga2026xl} & 100 \\
VRPTW & Solomon and Homberger~\cite{solomon1987vrptw,homberger1999two} & 356 \\
SDVRP & SET-1 through SET-4~\cite{chen2007split,archetti2008optimization} & 95 \\
\bottomrule
\end{tabular}
\caption{Public test sets reserved for final reporting.}
\label{tab:supp-public-tests}
\end{table}

\paragraph{Models and search budget.}
We run GLM-4-Flash and DeepSeek-V4-Flash separately. Within one backbone, every
LLM-based AHD method receives the same task description, template program,
training data, and a total allowance of 1,000 sampled programs (upper-level and lower-level searches combined for SpecAHD). We run each method independently three times. Each population-based run keeps ten programs and uses four sampler workers and four evaluator
workers. The complete embedding configuration uses
\texttt{multilingual-e5-large-instruct},
\texttt{jina-embeddings-v5-text-small}, and
\texttt{Qwen3-Embedding-0.6B}. The single-embedding ablation keeps only Jina.
Embedding models organize generation-time memory and are not loaded when a
frozen heuristic is evaluated.

\paragraph{Execution isolation.}
Generated upper programs pass a source guard that rejects imports, recursion,
system calls, and oversized source before execution. On Linux, each upper
candidate runs in a separate process with a 120-second limit. Generated lower
programs use the same per-candidate timeout configured by the task evaluator.
Malformed output, an empty validated region set, and a candidate-level timeout
remain visible in the run record. All generated programs are evaluated through
the same LLM4AD task and sandbox interface~\cite{liu2024llm4ad}.

\paragraph{Run selection and aggregation.}
A run is eligible once it reaches the prespecified sample budget. If a run
contains several restarted directories, the newest complete directory is used;
if none is complete, the directory with the largest recorded sample index is
retained and marked partial. This rule is applied before inspecting its final
score. Quality is summarized over complete runs, while completion counts and
wall-clock measurements retain partial runs and timeouts. The ablation plots
show all completed run-level points because three runs are insufficient for
a reliable significance test.

\paragraph{Evaluation unit and failure accounting.}
The lower-level and upper-level searches use different executable units. One lower-level
evaluation is a candidate--region response, whereas one upper-level evaluation runs
an upper program through the complete repair pipeline on one training
instance. A lower candidate enters repertoire selection only after evaluation
on the full task set exposed in that round. An upper candidate enters survivor
ranking only after all prescribed training instances finish. Parser errors,
timeouts, invalid local actions, rejected merges, and incomplete runs are not
imputed. They remain failures in the corresponding completion count.

\paragraph{Frozen evaluation state.}
Checkpoint selection precedes public testing. The chosen upper program, lower
repertoire, descriptor normalization statistics, and Repair Response Archive
are frozen together. Public instances are evaluated without archive write-back
and cannot change a later routing decision. The response router therefore uses
only records collected during training, including its fallback to a heuristic's
mean archived score when no observed response is available among the selected
neighbors.

\paragraph{Reported-cost convention.}
All routing objectives are minimized. Tables with values in thousands append
``k'' in the main text or state the scaling in the caption. A ``T'' in a
per-instance table means that no valid result was available for that method on
that instance. Such entries are displayed rather than replaced. The compact
public summary in the main paper compares SpecAHD with the lowest available
baseline cost on each reported instance. Its signed improvement is computed
before averaging, so a positive value always favors SpecAHD even though the
identity of the strongest baseline may vary across instances.

\paragraph{Scope of the statistical claims.}
The three search runs measure variability in stochastic program generation,
but they do not support a high-power significance test. We therefore report
run-level points, means, standard deviations, and completion status without a
claim of statistical significance. The public benchmark tables are descriptive
comparisons of the frozen checkpoint on the reported instances. They do not
enter checkpoint selection or any search update.

\paragraph{Task construction and replay order.}
The upper program is always evaluated from the baseline solution produced for
the current instance. Its raw region sequence first passes through the
problem-specific validator. Empty, repeated, overlapping, oversized, or
out-of-range supports are removed before a lower program is called. Retained
regions are processed in the validated order. After each accepted repair, the
descriptor for the next region is recomputed from the new feasible incumbent.
This detail matters because an early repair may change capacity slack,
endpoint distances, waiting time, or residual demand observed by a later
region. Replaying the same generated programs therefore requires the saved
validated order as well as the program source.

\paragraph{Archive observations and fallback.}
The archive stores checked candidate--task responses rather than one aggregate
score per program. An archive row contains the task identity, its normalized
descriptor, the lower-program identity, and the measured nonnegative repair
score. Not every retained program has been evaluated on every task exposed in
a later round. During lookup, the router averages only observed pairs among the
selected neighboring tasks. If none of those neighbors contains an observation
for a repertoire member, the router uses that member's mean archived score.
Ties are resolved by this mean and then by the fixed repertoire index. These
rules make a frozen evaluation deterministic once the baseline solution,
programs, archive, and validated region sequence have been fixed.

\paragraph{Run-level provenance.}
The stored audit trail links each reported search score to the generated source,
its design text, the variation operator, parent identifiers, selected skill
records, and the survivor population. Evaluator logs retain the instance or
repair-task key together with parser, timeout, validation, and checker status.
The analysis scripts choose a run directory by the prespecified completion rule
before reading its terminal objective. Consequently, reconstructing a table
does not require another language-model call and does not depend on manually
choosing among restarted directories after inspecting their scores.

\paragraph{What the timing columns include.}
Validation runtimes measure the complete frozen repair pipeline on one instance:
upper-program execution, validation, descriptor computation, archive lookup,
lower-program calls, merge, and complete-solution checks. They exclude search
training and embedding-model inference because embeddings are not loaded during
frozen validation. The full per-instance runtime fields remain in the evaluation
records; the appendix prints objective-only named-instance tables at 9-point
type so that all method columns remain legible without geometric scaling.

\subsection{Computational Environment and Artifacts}

\paragraph{Hardware.}
All experiments run on a single Ubuntu 22.04 server with one NVIDIA GPU (CUDA~12.x driver). The embedding server loads models onto \texttt{cuda:0} and listens at \texttt{http://127.0.0.1:8765}. Training and validation share the same physical machine; the embedding server starts once before the ablation suite and remains live for all runs. Frozen validation does not load embedding models and therefore does not require a GPU.

\paragraph{Software stack.}
The codebase targets Python~3.9--3.12 and is developed inside a Conda environment (\texttt{ahd}). The core numerical and machine-learning dependencies are \texttt{numpy}$\,<$\texttt{2.0.0}, \texttt{torch} (PyTorch), \texttt{scipy}, and \texttt{numba}. The routing baseline uses \texttt{pyvrp==0.13.4}. LLM calls go through \texttt{openai} (OpenAI-compatible client). Source-level code metrics rely on \texttt{tree-sitter-python==0.21} and \texttt{codebleu}. Logging and visualization use \texttt{tensorboard}, \texttt{wandb}, \texttt{matplotlib}, \texttt{pandas}, and \texttt{ttkbootstrap}. The full dependency list is frozen in \texttt{code/requirements.txt}.

\paragraph{LLM backbones and API.}
Two LLM backbones are evaluated independently: \texttt{glm-4-flash} (Zhipu, via \texttt{https://open.bigmodel.cn/api/paas/v4/}) and \texttt{deepseek-v4-flash} (DeepSeek, via \texttt{https://api.deepseek.com}). The API key is read from \texttt{api\_key.txt} (Zhipu) or \texttt{deepseek\_api.txt} (DeepSeek) and selected by the \texttt{SPECAHD\_API\_KEY\_FILE} environment variable. Every LLM call uses the OpenAI-compatible chat-completions interface with the default sampling parameters of each provider. The LLM client timeout is 120~seconds per request.

\paragraph{Embedding models.}
The generation-time memory uses three sentence-embedding models loaded through a local HuggingFace-compatible server: \texttt{intfloat/multilingual-e5-large-instruct}, \texttt{jinaai/jina-embeddings-v5-text-small}, and \texttt{Qwen/Qwen3-Embedding-0.6B}. Model weights are cached under \texttt{embedding\_model/} and served with \texttt{HF\_HUB\_OFFLINE=1} and \texttt{TRANSFORMERS\_OFFLINE=1}. The single-embedding ablation retains only the Jina model.

\paragraph{Fixed random seeds and data generation.}
All synthetic training and validation instances are generated with fixed seeds to ensure reproducibility across backbone and variant comparisons. TSP and CVRP dataset generation uses \texttt{np.random.seed(2024)} and \texttt{torch.manual\_seed(2024)}. Each synthetic instance identifier embeds its generation seed (e.g., \texttt{s\{seed\}} in the instance name). The independent held-out sets use disjoint seed ranges and distinct spatial distributions (e.g., clustered vs.\ uniform). Public TSPLIB, CVRP-XL, Solomon/Homberger, and SDVRP instances are never regenerated and are reserved for final testing only.

\paragraph{Search budget and concurrency.}
Every configuration uses the same fixed budget: 1,000 sampled programs, 1,000 generations, population size~10, 4~sampler workers, and 4~evaluator workers. We run every configuration independently three times with different random seeds for the LLM-based search. The upper candidate timeout is 120~seconds per instance on Linux; lower candidates share the same per-call timeout. On Windows, evaluation runs in-process with the same source-level guard.

\paragraph{Artifact availability.}
The frozen checkpoints (upper program, lower repertoire, descriptor statistics, and Repair Response Archive), the canonical-run audit trail, the label-correction mapping, and the validation scripts are released with the paper. The artifact repository includes \texttt{ablation\_scripts/}, \texttt{ablation\_results\_aaai/}, \texttt{best\_samples\_analysis/}, and \texttt{appendix\_e\_response\_routing/}. All run-level tables and figures are reproducible from the provided frozen artifacts and evaluation scripts without additional LLM calls.

\onecolumn
\section{D. Extended Benchmark Results}

\noindent\textbf{Abbreviations.}
The table headers use RS for RandSample, HC for HillClimb, FS for
FunSearch, PE for PartEvo, and MCTS for MCTS-AHD; EoH, EoH-S, ReEvo,
MLES, and LHNS retain their published method names~\cite{zhang2024importance,
romera2024mathematical,liu2024evolution,ye2024reevo,hu2026mles,
hu2025partevo,xie2025lhns,Zheng2025}. SpecAHD denotes the proposed method, and T marks
a timeout or unavailable result. Bold entries are the lowest costs for
each instance, including ties.

\subsection{Public Results by Named Instance}

Tables~\ref{tab:cvrp-named-test-results}--\ref{tab:sdvrp-named-test-results} report every completed named-instance cost used by the main summary: 79 CVRP, 30 TSP, 331 VRPTW, and 95 SDVRP rows. For each row, the baseline envelope is the lowest available numeric baseline; a method-level timeout remains marked ``T'' and is excluded only from that row's minimum. All printed rows contain a valid SpecAHD result.

\input{Word/results/benchmark_cvrp_named_test}
\input{Word/results/benchmark_tsp_named_test}
\input{Word/results/benchmark_vrptw_named_test}
\input{Word/results/benchmark_sdvrp_named_test}

\twocolumn

\section{E. Repair Specialization and Routing}

We evaluate whether the ten retained lower programs provide complementary repair
behavior and whether the frozen response router can exploit it. For each problem,
the selected upper program exposes regions on 60 training instances and a
disjoint set of 60 validation instances. Replaying every lower program yields
147 training and 142 held-out TSP regions, together with 431 training and 430
held-out CVRP regions. The two sets share no region keys. We fit the $k=5$
response router on training regions only and compare it with the global
training-best heuristic, uniform random selection, and a per-region hindsight
oracle. Uniform random selection uses 30 seeded repeats.

The main text reports the routing results; we explain here how the four metrics are computed from the underlying repair records.

\paragraph{Normalized response.}
For each region $r$, let $\mathcal H=\{h_1,\dots,h_{10}\}$ be the lower repertoire and let $c_{h,r}$ denote the complete-solution cost when heuristic $h$ repairs region $r$. The per-region response is linearly rescaled within $[0,1]$:
\[
\bar s_{h,r}=
\frac{\max_{h'\in\mathcal H}c_{h',r}-c_{h,r}}
{\max_{h'\in\mathcal H}c_{h',r}-\min_{h'\in\mathcal H}c_{h',r}}.
\]
A policy that selects heuristic $\pi(r)$ for region $r$ obtains the mean normalized response $\frac{1}{|\mathcal R|}\sum_{r\in\mathcal R}\bar s_{\pi(r),r}$, where $\mathcal R$ is the held-out region pool.

\paragraph{Regret.}
Let $c^{\star}_r=\min_{h\in\mathcal H}c_{h,r}$ be the per-region oracle cost. The oracle regret of a policy is the mean relative gap to this hindsight best:
\[
\text{Regret}=\frac{100}{|\mathcal R|}\sum_{r\in\mathcal R}
\frac{c_{\pi(r),r}-c^{\star}_r}{c^{\star}_r}.
\]

\paragraph{$\Delta$GB.}
Let $\bar h$ denote the single heuristic with the best average training score. The signed improvement of a policy over this global best is
\[
\Delta\text{GB}=\frac{100}{|\mathcal R|}\sum_{r\in\mathcal R}
\frac{c_{\bar h,r}-c_{\pi(r),r}}{c_{\bar h,r}}.
\]

\paragraph{Oracle hit rate.}
A policy ties the oracle on region $r$ exactly when $c_{\pi(r),r}=c^{\star}_r$. The oracle hit rate is the fraction of regions where this occurs.

\begin{table*}[tb]
\centering
\small
\begin{tabular}{@{}llrrrr@{}}
\toprule
Problem & Policy & Response $\uparrow$ & $\Delta$GB (\%) $\uparrow$ & Regret (\%) $\downarrow$ & Hit (\%) $\uparrow$ \\
\midrule
TSP & SpecAHD & 0.8695 & $+0.659$ & 2.032 & 54.93 \\
    & Global best & 0.8200 & 0.000 & 2.846 & 51.41 \\
    & Random & $0.7551{\pm}0.0248$ & $-1.371{\pm}0.502$ & $4.016{\pm}0.515$ & $47.98{\pm}2.96$ \\
    & Oracle & 1.0000 & $+2.552$ & 0.000 & 100.00 \\
\midrule
CVRP & SpecAHD & 0.9493 & $+0.188$ & 2.271 & 43.49 \\
     & Global best & 0.9489 & 0.000 & 2.139 & 45.58 \\
     & Random & $0.8038{\pm}0.0114$ & $-12.027{\pm}1.989$ & $14.465{\pm}2.038$ & $26.64{\pm}1.41$ \\
     & Oracle & 1.0000 & $+1.998$ & 0.000 & 100.00 \\
\bottomrule
\end{tabular}
\caption{Complete routing-policy comparison on held-out repair regions. The oracle is available only for retrospective analysis.}
\label{tab:supp-response-routing}
\end{table*}

\begin{table*}[tb]
\centering
\small
\begin{tabular}{@{}lrrrr@{}}
\toprule
Problem & Size & Best response $\uparrow$ & Regret (\%) $\downarrow$ & Coverage (\%) $\uparrow$ \\
\midrule
TSP  & 1  & 0.8695 & 2.032 & 54.93 \\
     & 2  & 0.9178 & 1.180 & 61.27 \\
     & 4  & 0.9666 & 0.502 & 76.76 \\
     & 6  & 0.9768 & 0.317 & 80.99 \\
     & 10 & 1.0000 & 0.000 & 100.00 \\
\midrule
CVRP & 1  & 0.9492 & 2.735 & 39.53 \\
     & 2  & 0.9824 & 0.842 & 69.07 \\
     & 4  & 0.9931 & 0.437 & 82.56 \\
     & 6  & 0.9989 & 0.066 & 96.74 \\
     & 10 & 1.0000 & 0.000 & 100.00 \\
\bottomrule
\end{tabular}
\caption{Held-out complementarity as the training-selected repertoire grows. Covered is the share of regions for which the subset contains at least one oracle-tied member.}
\label{tab:supp-repertoire-size}
\end{table*}

The benefit appears as soon as a second member enters the repertoire:
Table~\ref{tab:supp-repertoire-size} shows that oracle regret falls from $2.032\%$ to
$1.180\%$ on TSP and from $2.735\%$ to $0.842\%$ on CVRP. The gains continue as
additional members enter, although with diminishing returns. This behavior is
consistent with a repertoire whose programs cover different repair regions,
rather than a collection of interchangeable copies.

\section{F. Ablation Studies}

The two ablations isolate different parts of the system. The lower study evolves
and evaluates the repair repertoire under each configuration. The full-system
study freezes one lower repertoire for each problem and changes only the
mechanisms used to evolve the upper program. Every upper candidate is
still evaluated through region validation, response-based routing, local
construction, and complete-solution scoring. The latter is therefore an
end-to-end ablation of the upper-level search under controlled repair capacity, rather
than an upper program evaluated by itself.

\input{Word/ablation/paper/validation_ablation_combined}

\begin{table}[tb]
\centering
\small
\begin{tabular}{@{}lp{0.58\columnwidth}@{}}
\toprule
Variant & Change from the complete configuration \\
\midrule
Full SpecAHD & Three embeddings, semantic duplicate filtering, and skill-guided generation \\
Single embedding & Retain only the Jina embedding; keep filtering and skill guidance \\
w/o semantic filter & Keep all embeddings and skill guidance; admit semantic duplicates \\
w/o skill guidance & Keep all embeddings and filtering; remove skill records from generation prompts \\
\bottomrule
\end{tabular}
\caption{Controlled variants used in the ablation of the lower-level search and in the full-system ablation with a frozen repertoire.}
\label{tab:supp-ablation-variants}
\end{table}

\begin{figure}[tb]
\centering
\includegraphics[width=0.48\linewidth]{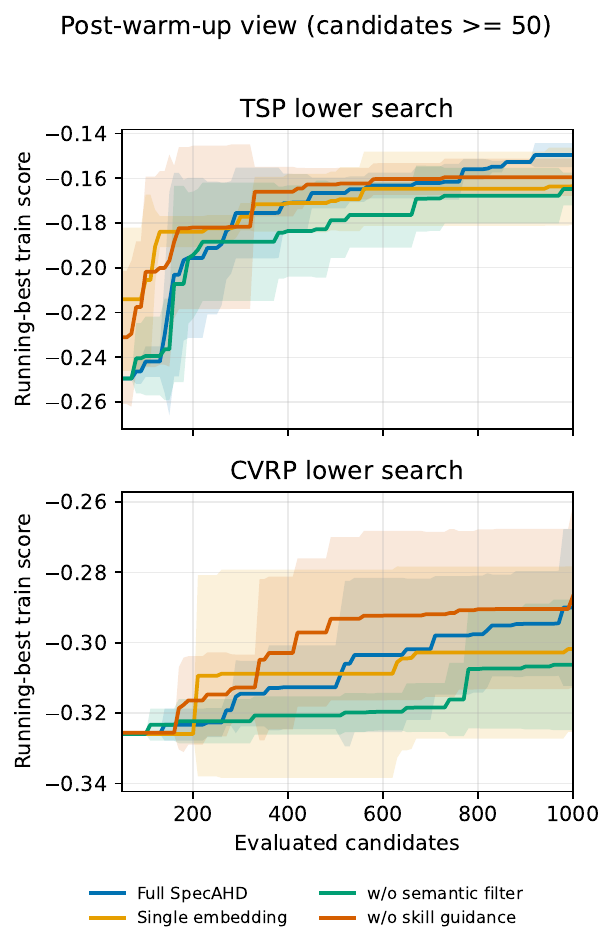}\hfill
\includegraphics[width=0.48\linewidth]{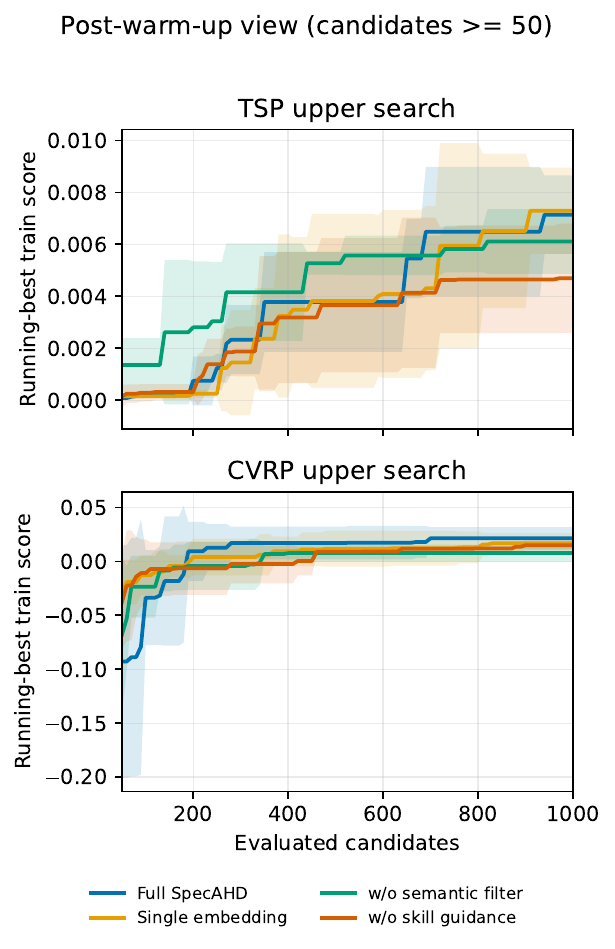}
\caption{Running-best training scores after the first 50 evaluated candidates. The lower-level search is shown on the left and the upper-level search on the right. Lines show the mean over three runs and shaded regions show one standard deviation. These curves describe optimization dynamics; held-out repair quality is reported separately.}
\label{fig:specahd-convergence}
\end{figure}

Figure~\ref{fig:specahd-convergence} exposes the lower-level and upper-level search trajectories behind the aggregate ablation results. These are training diagnostics rather than estimates of held-out solution quality.

The complete per-run means are reported in the ablation table of the main text. With three runs per setting, we do not perform a significance test. Per-instance evaluation records, completion counts, and runtimes are available in the released artifacts.

Execution robustness is summarized separately for the lower-level search and the full
system in Figure~\ref{fig:specahd-validation-runtime}. Removing semantic filtering
produces the longest complete TSP runs and the full-system CVRP timeout. The
single-embedding condition has one timeout in the CVRP lower-level search and one partial
full-system TSP run.

\begin{figure}[tb]
\centering
\includegraphics[width=0.48\linewidth]{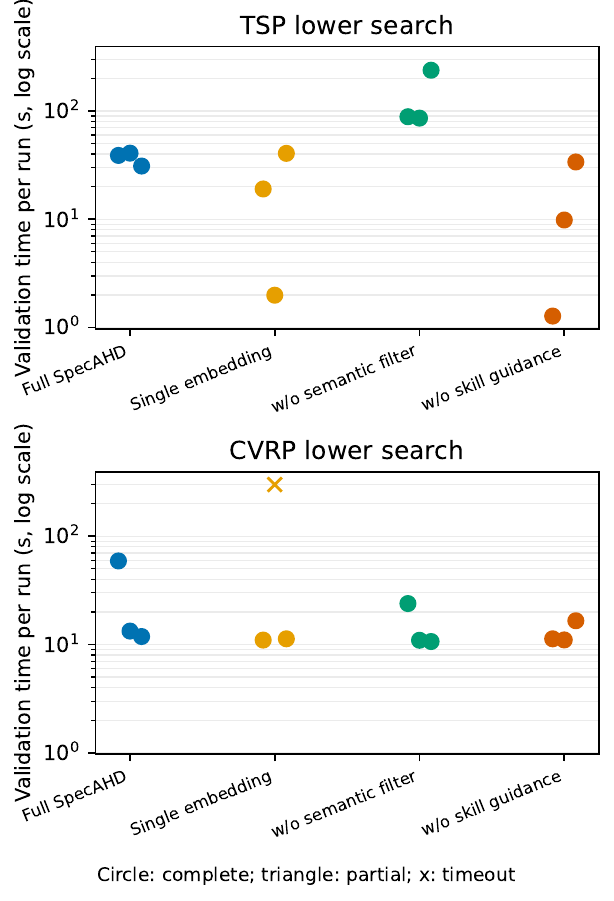}\hfill
\includegraphics[width=0.48\linewidth]{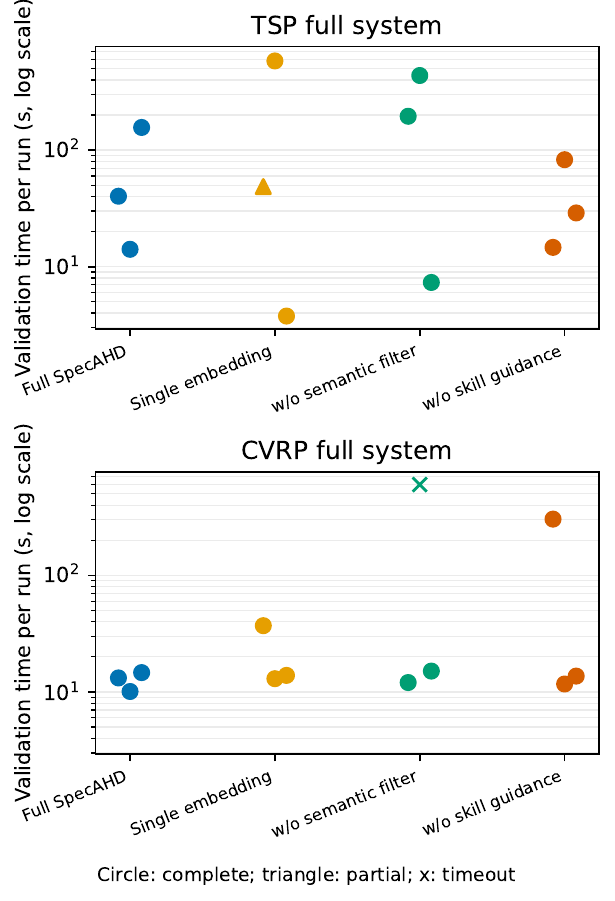}
\caption{Wall-clock validation time for the lower-level search (left) and full system (right) on 60 instances, shown on a logarithmic scale. Circles denote complete runs, triangles partial runs, and crosses timeouts. Embedding models are not loaded during validation, so the measurements reflect complete repair-pipeline execution rather than embedding inference.}
\label{fig:specahd-validation-runtime}
\end{figure}

\section{G. Prompts and Generation Records}

\paragraph{Fixed task contracts.}
The LLM may change the function body but not its name, arguments, or return
type. Upper-program prompts ask for regions rather than complete solutions. Lower-program
prompts ask for one constructive action from the supplied local state.
Table~\ref{tab:supp-program-contracts} summarizes the information exposed by
these interfaces; the code artifact contains their complete signatures and
task descriptions.

\begin{table}[tb]
\centering
\small
\begin{tabular}{@{}lp{0.25\columnwidth}p{0.48\columnwidth}@{}}
\toprule
Problem & Upper object & State visible to the lower program \\
\midrule
TSP & Route-order segments & Current and destination nodes, unvisited nodes, and distances \\
CVRP & Customer groups & Current node, depot, unvisited customers, remaining capacity, demands, and distances \\
VRPTW & Customer groups & CVRP state plus ready times, due times, and service times \\
SDVRP & Customer groups & Current node, depot, residual demand, remaining capacity, and distances \\
\bottomrule
\end{tabular}
\caption{Roles and state exposed by the immutable generation interfaces.}
\label{tab:supp-program-contracts}
\end{table}

For example, the TSP upper-program task instructs the model to return route-order
segments whose first and last nodes remain connected to the unchanged tour. It
explicitly rules out returning a complete tour and asks the program to expose
crossings, detours, density changes, or cluster boundaries that may benefit from
local reconstruction. The CVRP prompt likewise requests customer groups rather
than routes, with spatial coherence and capacity structure as available
signals. The corresponding validator, rather than the prompt, enforces the
hard interface.

\paragraph{Candidate-generation context.}
Mutation and recombination follow the executable parent-rewriting interface of
EoH~\cite{liu2024evolution}. A generation request contains the task contract,
one or more parent programs selected by the operator, their measured scores,
and the operator-specific instruction. The response must contain a natural
language design statement and one executable function. Text outside the
expected fields is discarded before evaluation.

\paragraph{Skill records.}
A retained child's record contains a short description of the code change, the
conditions under which its measured responses improved, observed failure
conditions, and the source program identifiers. Compatible records may be
inserted into a later mutation prompt. They do not alter a candidate score and
are not queried during final routing. The no-skill ablation sets the insertion
count to zero while preserving the same LLM, operators, budget, embeddings, and
semantic filter.

\paragraph{Saved audit trail.}
Each run stores sampled design text, generated source, score vector, operator
type, selected parent and skill identifiers, survivor populations, extracted
skill records, and the running-best score. These files permit candidate-level
reconstruction of the search trajectory without another LLM call.

\paragraph{Configuration identity.}
Synthetic-instance identifiers retain their generation seeds, while the run
index identifies an independent program-search execution. Some legacy launch
directories used batch labels that differ from the final paper variant names.
Before aggregation, a frozen mapping keyed only by search layer, problem, and
source label assigns the paper label. The audit record preserves both labels,
the run index, selected run directory, and correction note. Label assignment does
not inspect the terminal objective, and the same mapping is used for tables,
run-level plots, and completion counts.

\section{H. Representative Generated Programs}

The following excerpts are taken from the frozen TSP and CVRP artifacts used by
the full system. Docstrings and comments are shortened for space; the shown
selection expressions are unchanged. They illustrate the separation between an
upper program and a lower program.

\begin{lstlisting}[language=Python,caption={Representative TSP upper program.}]
def select_tsp_segments(route, coords, distance_matrix,
                        max_segments, max_segment_size):
    selected_segments = []
    for start in range(len(route) - max_segment_size):
        segment = route[start:start + max_segment_size]
        edges = [distance_matrix[segment[i]][segment[i + 1]]
                 for i in range(len(segment) - 1)]
        mean_edge = sum(edges) / len(edges)
        variance = sum((x - mean_edge) ** 2 for x in edges) / len(edges)
        inverse_density = len(segment) / sum(
            1 / (((coords[segment[k]][0] - coords[segment[k + 1]][0]) ** 2
                 + (coords[segment[k]][1] - coords[segment[k + 1]][1]) ** 2))
            for k in range(len(segment) - 1))
        selected_segments.append((segment, variance * inverse_density))
    selected_segments.sort(key=lambda item: item[1], reverse=True)
    return [segment for segment, _ in selected_segments[:max_segments]]
\end{lstlisting}

\begin{lstlisting}[language=Python,caption={Representative TSP lower program.}]
def select_next_node(current_node, destination_node,
                     unvisited_nodes, distance_matrix):
    visited = len(distance_matrix) - len(unvisited_nodes)
    log_base = 2.0
    if visited > 50:
        log_base = np.log(52) * 0.9
    elif visited >= 40:
        log_base *= 0.9
    penalty = 1.0
    if visited > 10:
        penalty = min(1 + distance_matrix[current_node][destination_node] / 100, 5)
    current = np.log(distance_matrix[current_node][unvisited_nodes] + 1)
    target = np.log(distance_matrix[destination_node][unvisited_nodes] + 1)
    score = current / log_base + penalty * (current - target)
    return unvisited_nodes[np.argmin(score)]
\end{lstlisting}

\begin{lstlisting}[language=Python,caption={Representative CVRP upper program.}]
def generate_customer_groups(coords, demands, capacity,
                             distance_matrix, max_groups, max_size):
    density = {i: sum(distance_matrix[i][j] * demands[j]
                      for j in range(len(coords)) if i != j)
               for i in range(len(coords))}
    unassigned = list(range(1, len(coords)))
    groups = []
    while unassigned and len(groups) < max_groups:
        seed = max(unassigned, key=lambda node: density[node])
        unassigned.remove(seed)
        group, load = [seed], demands[seed]
        while unassigned and len(group) < max_size:
            node = min(unassigned, key=lambda item: distance_matrix[seed][item])
            if load + demands[node] > capacity:
                break
            group.append(node)
            load += demands[node]
            unassigned.remove(node)
        groups.append(group)
    return groups
\end{lstlisting}

\begin{lstlisting}[language=Python,caption={Representative CVRP lower program.}]
def select_next_node(current_node, depot, unvisited_nodes,
                     rest_capacity, demands, distance_matrix):
    max_demand = np.max(demands[unvisited_nodes])
    factor = demands[current_node] / max_demand if max_demand > 0 else 0
    factor = max(0.5 / rest_capacity, min(0.5, factor), 0.1)
    capacity_ratio = rest_capacity / (max_demand + 1)
    savings = (distance_matrix[current_node][unvisited_nodes]
               - distance_matrix[depot][unvisited_nodes])
    cost = savings * (1 + factor) * capacity_ratio
    cost += distance_matrix[current_node][unvisited_nodes]
    cost *= 1 / (distance_matrix[current_node][unvisited_nodes] + 1)
    for index in np.argsort(cost):
        node = unvisited_nodes[index]
        if node != depot and rest_capacity >= demands[node]:
            return node
    return None
\end{lstlisting}

\bibliography{references}

%% file: Word/results/benchmark_cvrp_named_test.tex
\begingroup
\small
\captionof{table}{Named-instance CVRP objective values on CVRP-XL, in thousands; lower is better. Per-instance runtimes remain in the released evaluation records.}
\label{tab:cvrp-named-test-results}
\addtocounter{table}{-1}
\begin{longtable}[c]{@{}l@{\hspace{7pt}}*{10}{r@{\hspace{7pt}}}r@{}}
\toprule
Instance & RS & HC & FS & EoH & ReEvo & MLES & PE & LHNS & MCTS & EoH-S & SpecAHD \\
\midrule
\endfirsthead
\multicolumn{12}{c}{Table \thetable\ continued} \\
\toprule
Instance & RS & HC & FS & EoH & ReEvo & MLES & PE & LHNS & MCTS & EoH-S & SpecAHD \\
\midrule
\endhead
\midrule
\multicolumn{12}{r}{Continued on next page} \\
\endfoot
\bottomrule
\endlastfoot
XL-n1048-k237 &  500.93  &  485.87  &  502.78  &  493.92  &  455.57  &  508.88  &  494.90  &  605.20  &  486.40  & 459.31 &  \textbf{415.29} \\
XL-n1094-k157 &  119.05  &  129.79  &  119.12  &  121.14  &  126.81  &  155.52  &  154.36  &  281.72  &  120.08  & 119.12 &  \textbf{117.98} \\
XL-n1141-k112 &  107.82  &  108.95  &  107.10  &  110.05  &  120.30  &  159.74  &  217.64  &  316.94  &  110.52  & 107.23 &  \textbf{106.40} \\
XL-n1188-k96 &  160.19  &  162.65  &  161.55  &  154.50  &  157.24  &  182.45  &  255.99  &  390.84  &  155.31  & 138.12 &  \textbf{117.88} \\
XL-n1234-k55 &  113.23  &  125.25  &  113.54  &  116.02  &  124.11  &  168.21  &  159.10  &  301.44  &  115.26  & 111.10 &  \textbf{108.77} \\
XL-n1281-k29 &  44.89  &  45.55  &  45.15  &  43.08  &  50.74  &  80.17  &  123.95  &  202.55  &  43.78  & 42.34 &  \textbf{39.08} \\
XL-n1328-k19 &  50.47  &  52.64  &  50.36  &  49.93  &  60.87  &  98.81  &  81.12  &  242.49  &  49.92  & 49.47 &  \textbf{48.28} \\
XL-n1374-k278 &  325.95  &  331.52  &  326.77  &  326.23  &  323.51  &  384.45  &  371.13  &  545.25  &  328.63  & 304.24 &  \textbf{270.53} \\
\end{longtable}
\endgroup

%% file: Word/results/benchmark_tsp_named_test.tex
\begingroup
\small
\captionof{table}{Named-instance TSP objective values on TSPLIB, in thousands; lower is better. Per-instance runtimes remain in the released evaluation records.}
\label{tab:tsp-named-test-results}
\addtocounter{table}{-1}
\begin{longtable}[c]{@{}l@{\hspace{7pt}}*{10}{r@{\hspace{7pt}}}r@{}}
\toprule
Instance &  RS  &  HC  &  FS  &  EoH  &  ReEvo  &  MLES  &  PE  &  LHNS  &  MCTS  &  EoH-S  &  SpecAHD \\
\midrule
\endfirsthead
\multicolumn{12}{c}{Table \thetable\ continued} \\
\toprule
Instance &  RS  &  HC  &  FS  &  EoH  &  ReEvo  &  MLES  &  PE  &  LHNS  &  MCTS  &  EoH-S  &  SpecAHD \\
\midrule
\endhead
\midrule
\multicolumn{12}{r}{Continued on next page} \\
\endfoot
\bottomrule
\endlastfoot
pr1002 &  339.68  &  349.40  &  349.40  &  325.56  &  331.21  &  333.18  &  326.07  &  1072.77  &  332.14  & 310.03 & \textbf{292.34} \\
u1060 &  272.31  &  261.28  &  \textbf{260.17}  &  291.87  &  290.03  &  286.28  &  299.19  &  762.90  &  283.81  & 270.20 &  \textbf{260.17} \\
vm1084 &  3663.49  &  5350.74  &  5350.74  &  298.95  &  306.68  &  1982.05  &  295.48  &  2840.10  &  303.20  & 291.01 &  \textbf{273.19} \\
pcb1173 &  107.06  &  123.84  &  123.84  &  71.44  &  72.12  &  88.83  &  74.55  &  290.38  &  72.71  & 66.90 &  \textbf{65.96} \\
d1291 &  121.21  &  150.85  &  150.85  &  60.47  &  63.17  &  90.44  &  60.50  &  330.84  &  64.67  & 60.02 & \textbf{59.27} \\
rl1304 &  2262.13  &  3231.69  &  3231.69  &  324.09  &  329.65  &  1293.46  &  323.80  &  2887.49  &  335.44  & 305.05 & \textbf{291.95} \\
rl1323 &  2171.78  &  3088.19  &  3088.19  &  330.00  &  352.14  &  1248.39  &  333.29  &  2708.96  &  350.40  & 326.68 & \textbf{323.94} \\
nrw1379 &  498.13  &  712.34  &  712.34  &  70.32  &  70.78  &  283.18  &  69.72  &  452.16  &  69.96  & 69.37 & \textbf{65.14} \\
\end{longtable}
\endgroup

%% file: Word/results/benchmark_vrptw_named_test.tex
\begingroup
\small
\captionof{table}{Named-instance VRPTW objective values on Solomon/Homberger; lower is better.}
\label{tab:vrptw-named-test-results}
\addtocounter{table}{-1}
\begin{longtable}[c]{@{}lrrrrrrrrr@{}}
\toprule
Instance & RS & HC & EoH & ReEvo & MLES & PE & LHNS & MCTS & SpecAHD \\
\midrule
\endfirsthead
\multicolumn{10}{c}{Table \thetable\ continued} \\
\toprule
Instance & RS & HC & EoH & ReEvo & MLES & PE & LHNS & MCTS & SpecAHD \\
\midrule
\endhead
\midrule
\multicolumn{10}{r}{Continued on next page} \\
\endfoot
\bottomrule
\endlastfoot
C1\_10\_1 & 55.71 & 176.09 & 83.18 & 56.09 & 73.68 & 57.78 & 136.32 & 73.70 & \textbf{43.31} \\
C1\_10\_10 & 55.05 & 180.68 & 52.72 & 54.86 & 53.72 & 53.00 & 120.66 & 52.89 & \textbf{44.95} \\
C1\_10\_2 & 63.52 & 194.06 & 75.01 & 71.10 & 66.95 & 71.16 & 144.15 & 73.09 & \textbf{44.57} \\
C1\_10\_3 & 64.09 & 227.35 & 62.19 & 68.00 & 62.16 & 66.00 & 143.62 & 65.35 & \textbf{45.04} \\
C1\_10\_4 & 57.64 & 245.98 & 53.15 & 58.84 & 55.20 & 57.64 & 133.79 & 54.91 & \textbf{43.73} \\
C1\_10\_5 & 56.23 & 174.13 & 74.15 & 56.85 & 71.30 & 58.03 & 126.98 & 69.48 & \textbf{43.74} \\
C1\_10\_6 & 61.94 & 173.38 & 73.08 & 68.09 & 73.43 & 68.61 & 130.33 & 71.28 & \textbf{44.90} \\
C1\_10\_8 & 62.88 & 169.11 & 61.73 & 62.55 & 61.88 & 62.90 & 122.63 & 63.16 & \textbf{44.59} \\
\end{longtable}
\endgroup

%% file: Word/results/benchmark_sdvrp_named_test.tex
\begingroup
\small
\captionof{table}{Named-instance SDVRP objective values on SET-1--4; lower is better.}
\label{tab:sdvrp-named-test-results}
\addtocounter{table}{-1}
\begin{longtable}[c]{@{}lrrrrrrrrr@{}}
\toprule
Instance & RS & HC & FS & EoH & MLES & PE & LHNS & MCTS & SpecAHD \\
\midrule
\endfirsthead
\multicolumn{10}{c}{Table \thetable\ continued} \\
\toprule
Instance & RS & HC & FS & EoH & MLES & PE & LHNS & MCTS & SpecAHD \\
\midrule
\endhead
\midrule
\multicolumn{10}{r}{Continued on next page} \\
\endfoot
\bottomrule
\endlastfoot
SD1 & 29.66 & 27.38 & 30.37 & 29.66 & 29.66 & 29.66 & 29.66 & 27.38 & \textbf{24.00} \\
SD10 & 308.21 & 301.41 & 294.46 & 298.87 & 300.72 & 309.74 & 310.48 & 301.36 & \textbf{288.14} \\
SD11 & \textbf{1328.00} & 1332.00 & 1876.62 & \textbf{1328.00} & \textbf{1328.00} & \textbf{1328.00} & 1332.00 & 1330.67 & 1428.00 \\
SD12 & 768.00 & 756.55 & 873.82 & 775.08 & 756.63 & 796.33 & 782.83 & \textbf{752.14} & 760.00 \\
SD13 & 1047.80 & 1045.92 & 1213.38 & 1054.46 & 1057.89 & 1051.80 & 1055.80 & \textbf{1043.00} & 1081.99 \\
SD14 & 1131.20 & \textbf{1117.93} & 1212.71 & 1143.69 & 1124.27 & 1145.70 & 1144.49 & 1124.95 & 1141.10 \\
SD15 & 1597.78 & \textbf{1551.57} & 1706.26 & 1588.15 & 1576.88 & 1584.04 & 1584.93 & 1556.90 & 1623.55 \\
SD16 & 357.56 & 355.44 & \textbf{344.57} & 356.62 & 352.05 & 357.45 & 357.18 & 354.76 & 370.64 \\
\end{longtable}
\endgroup

%% file: Word/ablation/paper/validation_ablation_combined.tex
\begin{table}[tb]
\centering
\small
\begin{tabular}{@{}lrrrr@{}}
\toprule
& \multicolumn{2}{c}{Lower-level search} & \multicolumn{2}{c}{Frozen repertoire} \\
\cmidrule(lr){2-3}\cmidrule(l){4-5}
Variant & TSP $\uparrow$ & CVRP $\downarrow$ & TSP $\uparrow$ & CVRP $\uparrow$ \\
& (\%) & (k) & (\%) & (\%) \\
\midrule
Full & \textbf{$5.658$} & \textbf{$59.943$} & \textbf{$-0.337$} & \textbf{$0.578$} \\
Single embedding & $-0.015$ & $83.512$ & $-1.857$ & $-0.350$ \\
w/o semantic filter & $3.783$ & $60.712$ & $-0.986$ & $-1.304$ \\
w/o skill guidance & $4.227$ & $118.517$ & $-1.141$ & $-0.079$ \\
\bottomrule
\end{tabular}
\caption{Held-out ablations of the lower-level search and the full system with a frozen repertoire. Each setting uses three independent runs; cells average completed runs. Appendix~F gives run-level values, sample standard deviations, and completion status.}
\label{tab:eohs-validation-lower-ablation}
\label{tab:eohs-validation-upper-ablation}
\end{table}